\documentclass{article}
\usepackage{ijcai25}

\usepackage{times}
\usepackage{soul}
\usepackage{url}
\usepackage[hidelinks]{hyperref}
\usepackage[utf8]{inputenc}
\usepackage[small]{caption}
\usepackage{graphicx}
\usepackage{amsmath}
\usepackage{amssymb}                 % Additional math symbols
\usepackage{amsthm}
\usepackage{booktabs}
\usepackage{algorithm}
\usepackage{algorithmic}
\usepackage[switch]{lineno}
\usepackage{xcolor}

\newtheorem{example}{Example}

\usepackage{thmtools}
\usepackage{thm-restate}

\declaretheorem[name=Definition]{definition}

\usepackage{enumitem}
\setlist[1]{itemsep=0.5em}
\setlist{topsep=0.5em}
\setlist{leftmargin=3em}

\usepackage[font=small,labelfont=bf,textfont=rm]{caption}
\usepackage{subcaption}

\newenvironment{pproof}
    {\par\vspace*{0.5ex}\noindent\rm {\bf Proof:} }
    {\par\vspace*{1ex}\hfill\tiny$\Box$}

\DeclareFontFamily{U}{matha}{\hyphenchar\font45}
\DeclareFontShape{U}{matha}{m}{n}{
      <5> <6> <7> <8> <9> <10> gen * matha
      <10.95> matha10 <12> <14.4> <17.28> <20.74> <24.88> matha12
      }{}
\DeclareSymbolFont{matha}{U}{matha}{m}{n}
\DeclareMathSymbol{\odiv}         {2}{matha}{"63}

\makeatletter
\newcommand*{\bigCup}{\mathop{\mathpalette\big@Cup\relax}\slimits@}

\newcommand*{\big@Cup}[2]{%
   \setbox\z@=\hbox{\m@th$#1\Cup$}%
   \setbox\z@=\vtop{\vbox{\kern.2\ht\z@\copy\z@}\kern.1\ht\z@}%
   \setbox\tw@=\hbox{\m@th$#1\bigcup$}%
   \vcenter{\hbox{\resizebox{!}{1.4\ht\tw@}{\box\z@}}}%
}

\makeatother

\newcommand{\twiddle}{\mathrel|\joinrel\sim} % for conditionals

\newcommand{\mods}[1]{[\![#1]\!]}
\newcommand{\bel}[1]{[#1]}
\newcommand{\cbel}[1]{[#1]_{>}}

\newcommand{\contract}{\div} 

\newcommand{\STQ}{\oplus_{\mathrm{STQ}}}

\newcommand{\CConR}[1]{$(\mathrm{C}{#1}^{\scriptscriptstyle \contract}_{\scriptscriptstyle \preccurlyeq})$}
\newcommand{\CConRn}[1]{$(\mathrm{C}{#1}^{\scriptscriptstyle \odiv}_{\scriptscriptstyle \preccurlyeq})$}

\newcommand{\CConS}[1]{$(\mathrm{C}{#1}_{\scriptscriptstyle\mathrm{b}}^{\scriptscriptstyle \contract})$}

\newcommand{\KCon}[1]{$(\mathrm{K}{#1}^{\scriptscriptstyle \contract})$}
\newcommand{\KConP}[1]{$(\mathrm{K}{#1}^{\scriptscriptstyle \odiv})$}

\newcommand{\HI}{$(\mathrm{HI})$}

\newcommand{\MultiConInter}{$(\mathrm{Int}_{\scriptscriptstyle\mathrm{b}}^{\odiv})$}
\newcommand{\MultiConAggregPrec}{$(\mathrm{Agg}^{\odiv}_{\scriptscriptstyle \preccurlyeq})$}

\newcommand{\LBO}{$(\mathrm{LB}_{\scriptscriptstyle\mathrm{b}}^{\scriptscriptstyle \oplus})$}

\newcommand{\LBConBel}{$(\mathrm{LB}_{\scriptscriptstyle\mathrm{b}}^{\odiv})$}

\newcommand{\UBO}{$(\mathrm{UB}_{\scriptscriptstyle\mathrm{b}}^{\scriptscriptstyle \oplus})$}
\newcommand{\UBOConBel}{$(\mathrm{UB}_{\scriptscriptstyle\mathrm{b}}^{\odiv})$}

\newcommand{\SB}{$(\mathrm{PAR}_{\scriptscriptstyle \min})$}

\newcommand{\SBConBel}{$(\mathrm{PAR}_{\scriptscriptstyle\mathrm{b}}^{\odiv})$}

\newcommand{\PPAR}{$(\mathrm{PAR}^{\scriptscriptstyle \oplus}_{\scriptscriptstyle \preccurlyeq})$}
\newcommand{\PPARMin }{$(\mathrm{PAR}^{\scriptscriptstyle \oplus}_{\scriptscriptstyle \min})$}

\newcommand{\FacPref}{$(\mathrm{F}^{\oplus}_{\preccurlyeq})$}

\newcommand{\SPU}{$(\mathrm{SPU}^{\oplus}_{\preccurlyeq})$}

\newcommand{\SPUPlus}{$(\mathrm{SPU+}^{\oplus}_{\preccurlyeq})$}

\newcommand{\WPU}{$(\mathrm{WPU}^{\oplus}_{\preccurlyeq})$}

\newcommand{\WPUPlus}{$(\mathrm{WPU+}^{\oplus}_{\preccurlyeq})$}

\newcommand{\FacMin}{$(\mathrm{F}^{\oplus}_{\scriptscriptstyle \min})$}

\newcommand{\FacConBel}{$(\mathrm{F}_{\scriptscriptstyle\mathrm{b}}^{\odiv})$}

\newcommand{\Cn}{\mathrm{Cn}}

\newcommand{\CRat}{\mathrm{Cl_{rat}}}

\begin{document}

\title{Parallel Belief Contraction via\\ Order Aggregation}

\author{Jake Chandler$^1$ \and Richard Booth$^2$
\affiliations
$^1$La Trobe University\\
$^2$Cardiff University\\
\emails
Jake.Chandler@cantab.net,
BoothR2@cardiff.ac.uk
}

%\author{Anonymous}

\date{}

\maketitle

\begin{abstract}
The standard ``serial'' (aka ``singleton'') model of belief contraction models the manner in which an agent's corpus of beliefs responds to the removal of a single item of information.  One salient extension of this model introduces the idea of ``parallel'' (aka ``package'' or ``multiple'') change, in which an entire set of items of information are  simultaneously  removed. Existing research on the latter has largely focussed on single-step parallel contraction: understanding the behaviour of beliefs after a single parallel contraction. It has also focussed on generalisations to the parallel case  of serial contraction operations whose characteristic properties are extremely weak. Here we consider how to extend serial contraction operations that obey stronger properties. Potentially more importantly, we also consider the iterated case: the behaviour of beliefs after a sequence of parallel  contractions. We propose a general method for extending serial iterated belief change operators to handle parallel change based on an $n$-ary generalisation of  Booth \& Chandler's TeamQueue binary order aggregators.
\end{abstract}

\section{Introduction}

The field of belief revision studies the formal rationality constraints that govern the impact of the removal or addition of particular beliefs on an agent's broader world view. The incorporation of new beliefs is modelled by an operation of ``revision'', while the removal of beliefs is modelled by an operation of ``contraction''.

Initial work in this area was restricted to studying the repercussions of (i) a {\em single} episode of change (single-step change), involving the removal or addition of (ii) a {\em single} item of information (serial change). In this narrow context, the AGM postulates presented in \cite{alchourron1985logic} are widely accepted to provide adequate constraints on both contraction and revision, although belief change operations whose characteristic axioms fall considerably short of full AGM have also been studied extensively, including
 serial partial meet contraction \cite{alchourron1985logic}, serial partial meet base contraction \cite{Hansson1992aa} and serial kernel contraction \cite{KCH}.

The focus was later broadened. Two new aspects were considered: (iii) the behaviour of beliefs under {\em successive} changes (iterated change), and (iv) their response to the simultaneous removal or addition of {\em multiple} items of information (parallel change). With the exceptions of \cite{DelgrandeJames2012PbrR}, which focuses on revision, and \cite{SpohnPC},  which tackles contraction, these generalisations have largely been carried out separately, with research focusing either on iterated serial change or on single-step parallel change.

Work on iterated serial change notably saw the introduction of the postulates of Darwiche \& Pearl \cite{darwiche1997logic} for iterated serial revision and the postulates of Chopra {\em et al} \cite{chopra2008iterated} for iterated serial contraction, as well as various strengthenings thereof.

Regarding single-step parallel change, single-step parallel revision has been plausibly claimed to reduce to single-step serial revision (see \cite{DelgrandeJames2012PbrR}). Work on single-step parallel change has therefore focussed on the less obvious case of contraction. For reasons that are not entirely clear, however, the emphasis here has been on extending to the parallel case serial contraction operations that do {\em not} satisfy full AGM. We find proposals for partial meet parallel contraction (see \cite{https://doi.org/10.1111/j.1755-2567.1989.tb00725.x}, \cite{FurHanSMC} and \cite{DBLP:journals/jphil/ReisF12}), with an interesting special case studied in \cite{DBLP:journals/jphil/FermeR12}, \cite{DBLP:journals/rsl/FermeR13} and \cite{DBLP:journals/amai/ReisPF16}; there also exists an extension to the parallel case of serial kernel contraction (see \cite{DBLP:journals/sLogica/FermeSS03}). In contrast,  little attention has been paid to extending fully AGM-compliant operations.

This article aims to fill a substantial gap by extending fully AGM-compliant serial contraction not only to the single-step parallel case but to the iterated parallel case as well. It achieves this goal by employing a generalisation to the $n$-ary case of a binary order aggregation method--``TeamQueue'' aggregation--proposed in another context by Booth \& Chandler \cite{DBLP:journals/ai/BoothC19}. An axiomatic characterisation of this generalisation is provided, which will be of interest independently of the question of parallel change.

The plan of the paper is as follows. In Section \ref{sec:PrincBelCh}, we recapitulate basic notions of serial belief contraction, both single-step and iterated. Section \ref{sec:PrincParaBelCh} turns to the parallel case, restricting attention to single-step parallel contraction due to the absence of relevant work on the iterated case. There, we show that a particularly plausible approach to this issue, the ``intersective'' approach, validates a number of plausible principles due to Furhmann \& Hansson's, as well as two further ones that we introduce here. This is an important result, since, collectively, these principles generalise to the parallel case the AGM postulates for serial contraction. In Section \ref{sec:TQAgg}, we outline and discuss the $n$-ary generalisation of TeamQueue aggregation, covering its construction, semantic and syntactic characterisations, noteworthy properties, and connection to Lehmann \& Magidor's concept of rational closure. Section \ref{sec:TQPara} then puts $n$-ary TeamQueue aggregation to work in the construction of iterated parallel contraction operators, generalising the intersective approach to the iterated case. Section \ref{sec:Concl} concludes with open questions and suggestions for future research, including the potential broader applications of TeamQueue aggregation. In order to improve readability, proofs of propositions and theorems have been relegated to the \hyperref[s:Appendix]{appendix}.

\section{Serial belief contraction} 
\label{sec:PrincBelCh}

In the standard model of belief change, the beliefs of an agent are represented by a {\em belief state} $\Psi$. The latter determines a {\em belief set} $\bel{\Psi}$, a deductively closed set of sentences, drawn from a  propositional, truth-functional, finitely-generated language $L$. The set of classical logical consequences of $S\subseteq L$ will be denoted by $\mathrm{Cn}(S)$. When $S$ is simply the singleton set $\{C\}$, we write $\Cn(C)$. The set of $2^{n}$ propositional worlds or valuations will be denoted by $W$, and the set of models of a given sentence $A$ by $\mods{A}$. 

The core of this model includes two ``serial'' belief change operations, revision $\ast$ and contraction $\contract$, both mapping a pair consisting of a state and a single input sentence onto a state. Revision models the  incorporation of the input into the agent's beliefs, while contraction models its removal. While earlier discussions of the model focussed on single-step serial belief change, i.e.~the change brought about by a single episode of revision or contraction by a single sentence, attention shifted to iterated serial change, involving a succession of episodes of serial revision or contraction.

\subsection{Single-step serial change}
\label{ss:SSSC}

In the case of single-step serial belief change, it is widely accepted that the AGM postulates, introduced in \cite{alchourron1985logic}, provide an adequately strong set of rationality constraints. In relation to contraction, these are:
\begin{tabbing}
    \=BLAHBLAII\=\kill

\> \KCon{1}  \> $\Cn([\Psi\contract A])\subseteq [\Psi\contract A]$\\[0.1cm]

\> \KCon{2}  \> $[\Psi\contract  A]\subseteq[\Psi]$\\[0.1cm]

\> \KCon{3} \> If $A\notin [\Psi]$, then $[\Psi\contract A]=[\Psi]$\\[0.1cm]

\> \KCon{4} \> If $A\notin \Cn(\varnothing)$, then $A\notin [\Psi\contract A]$\\[0.1cm]

\> \KCon{5}  \>  If $A\in[\Psi]$, then $[\Psi]\subseteq\Cn([\Psi\contract A]\cup\{ A\})$\\[0.1cm]

\> \KCon{6}  \> If $\Cn(A)=\Cn(B)$, then $[\Psi\contract A]=[\Psi\contract B]$\\[0.1cm]

\> \KCon{7} \> $[\Psi\contract A]\cap[\Psi\contract B]\subseteq [\Psi\contract A\wedge B]$\\[0.1cm]

\> \KCon{8}  \> If $A\notin [\Psi\contract A\wedge B]$, then $[\Psi\contract A\wedge B]\subseteq$\\
\> \> $[\Psi\contract A]$\\[-0.25em]
\end{tabbing} 
\vspace{-1em}

\noindent The first six principles are known as the ``basic'' AGM postulates. The last two are known as the ``supplementary'' ones. Analogous principles regulate single-step serial revision. We call a serial contraction operator that satisfies \KCon{1}--\KCon{8} an AGM contraction operator.

A principle known as the Harper identity \cite{harper1976rational} allows us to define single-step serial contraction in terms of single-step serial revision.

\begin{tabbing}
    \=BLAHBLAII\=\kill
\>\HI  \> $\bel{\Psi\contract A}= \bel{\Psi} \cap \bel{\Psi\ast \neg A}$ \\[-0.25em]
\end{tabbing}
\vspace{-1em}

\noindent The motivation for this principle is straightforward. The idea is that, in contracting by $A$, we are opening our minds to the possibility that $A$ is false. So we must retract anything that would be no longer endorsed, had one come to believe that this possibility is an actuality. This, however, is the only modification to our prior beliefs that we should make, as we should retract nothing further and introduce nothing new.

A representation theorem connects contraction operators compliant with the full set of AGM postulates to total preorders (TPOs), i.e. reflexive, complete and transitive binary relations, over sets of propositional worlds. More specifically each $\Psi$ can be associated with  a  TPO  $\preccurlyeq_\Psi$ over $W$, such that   $\min(\preccurlyeq_{\Psi\contract A},W) = \min(\preccurlyeq_\Psi, W)\cup\min(\preccurlyeq_\Psi, \mods{\neg A})$ (see \cite{10.1016/j.ijar.2016.06.010}).

The information conveyed by the TPOs associated with belief states can be equivalently captured by conditional belief sets  $\cbel{\Psi}:=\{A>B\mid B\in \bel{\Psi\ast A}\}$ or again nonmonotonic consequence relations  $\twiddle_\Psi=\{\langle A, B\rangle\mid A > B\in \cbel{\Psi} \}$.  The AGM postulates ensure that  such belief sets or consequence relations are ``rational'' (in the sense of \cite{lehmann1992does}) and ``consistency preserving'' (see \cite{10.1007/BFb0018421}).

These results mean that the various principles that we shall be discussing can typically be presented in several equivalent alternative formats, where we will use subscripts to distinguish between these, with the non-subscripted version of the name generically referring to the principle regardless of presentation. The names of principles framed in terms of TPOs will be subscripted with $\preccurlyeq$. It will sometimes be useful to present principles in terms of minimal sets, denoting the $\preccurlyeq$-minimal subset of $S\subseteq W$, that is $\{x\in S\mid \forall y\in S, x\preccurlyeq y\}$, by $\min(\preccurlyeq, S)$. We will use the subscript $\min$ to indicate presentation in this format. Similarly, a principle cast in terms of conditional belief sets will be subscripted with $>$. Where required for disambiguation, the names of principles presented in terms of belief sets will include the subscript $\mathrm{b}$. Superscripts will be used to indicate the particular operation, such as $\ast$ or $\contract$, whose behaviour a given postulate constrains.

\subsection{Iterated serial contraction}
\label{ss:ItSer}

 When it comes to sequences of serial contraction, the basic postulates of Chopra {\em et al } \cite{chopra2008iterated} remain largely uncontroversial. While these have been supplemented in various ways, few additions have been uncontested and we shall not be discussing them here.  In the case of contraction, supplementary postulates have yielded moderate, priority (see \cite{nayak2007iterated}), and restrained (introduced in \cite{DBLP:journals/jphil/ChandlerB23}) contraction operators. But these operations are alike in identifying, for the purposes of belief change, belief states with TPOs, a view criticised in \cite{DBLP:journals/jphil/BoothC17}.

The postulates of Chopra {\em et al } can be presented either ``syntactically'' in terms of belief sets or ``semantically'' in terms of TPOs. Syntactically, they are given by: 
\begin{tabbing}
    \=BLAHBLAII\=\kill

\> \CConS{1}  \>  If $\neg A \in \mbox{Cn}(B)$ then $\bel{(\Psi \contract A) \ast B} = \bel{\Psi \ast B}$ \\[0.1cm]

\> \CConS{2} \> If $A \in \mbox{Cn}(B)$ then $\bel{(\Psi \contract A) \ast B} = \bel{\Psi \ast B}$ \\[0.1cm]

\> \CConS{3}   \> If $\neg A \in \bel{\Psi \ast B}$ then $\neg A \in \bel{(\Psi \contract A) \ast B}$ \\[0.1cm]

\> \CConS{4}  \> If $A \not\in \bel{\Psi \ast B}$ then $A \not\in \bel{(\Psi \contract A) \ast B}$\\[-0.25em]
\end{tabbing} 
\vspace{-1em}
\noindent and semantically by:
\begin{tabbing}
    \=BLAHBLAII\=\kill

\> \CConR{1}  \>  If $x,y \in \mods{\neg A}$ then $x \preccurlyeq_{\Psi \contract A} y$ iff  $x \preccurlyeq_\Psi y$ \\[0.1cm]

\> \CConR{2}  \> If $x,y \in \mods{A}$ then $x \preccurlyeq_{\Psi \contract A} y$ iff  $x \preccurlyeq_\Psi y$\\[0.1cm]

\> \CConR{3}   \> If $x \in \mods{\neg A}$, $y \in \mods{A}$ and $x \prec_\Psi y$ then\\
\> \>  $x \prec_{\Psi \contract A} y$ \\[0.1cm]

\> \CConR{4} \> If $x \in \mods{\neg A}$, $y \in \mods{A}$ and $x \preccurlyeq_\Psi y$ then\\
\> \>  $x \preccurlyeq_{\Psi \contract A} y$\\[-0.25em]
\end{tabbing} 
\vspace{-1em}

\noindent The question of how to extend the Harper Identity to the iterated case was considered in  \cite{DBLP:journals/ai/BoothC19}. We briefly recapitulate this contribution here, since it is relevant to what follows. In that paper, it was first noted that the naive suggestion of simply recasting \HI~in terms of conditional belief sets.

\begin{tabbing}
    \=BLAHBLAII\=\kill

\> (NiHI$^{\scriptscriptstyle \contract}_{\scriptscriptstyle >}$) \> $\cbel{\Psi \contract A}=\cbel{\Psi} \cap \cbel{\Psi \ast \neg A}$\\[-0.25em]
\end{tabbing} 
\vspace{-1em}

\noindent (equivalently, in terms of non-conditional belief sets: $\bel{(\Psi \contract A) \ast B}=\bel{\Psi \ast B} \cap \bel{(\Psi \ast \neg A) \ast B}$) is a non-starter: on pains of placing undue restrictions on the space of permissible conditional belief sets, the left-to-right inclusion in the naive suggestion was shown to be jointly inconsistent with several of the AGM postulates for serial revision and contraction.

A proposal was then made, involving a binary TPO aggregation function $\oplus$, mapping pairs of input TPOs onto a single aggregate output TPO:

\begin{tabbing}
    \=BLAHBLAII\=\kill

\> (iHI$^{\scriptscriptstyle \contract}_{\scriptscriptstyle \preccurlyeq}$) \> $\preccurlyeq_{\Psi\contract A} = \oplus\langle\preccurlyeq_{\Psi}, \preccurlyeq_{\Psi\ast \neg A}\rangle$\\[-0.25em]
\end{tabbing} 
\vspace{-1em}

\noindent A family of binary aggregators, the ``TeamQueue'' (TQ) family, was argued to be appropriate for this job, with one specific member of this family, the ``Synchronous TeamQueue'' function $\STQ$, being singled out as particularly promising. It was shown that, when $\oplus$ is taken to be a TQ aggregation function, (iHI$^{\scriptscriptstyle \contract}_{\scriptscriptstyle \preccurlyeq}$) allows for the derivation of several important principles, including \HI, which comes out as a special case of (iHI$^{\scriptscriptstyle \contract}_{\scriptscriptstyle \preccurlyeq}$), as well as  \CConS{1} to  \CConS{4} above, which are derivable from the corresponding Darwiche-Pearl postulates for iterated serial revision. 

Taking $\oplus$ to specifically correspond to $\STQ$ was argued to yield further desirable theoretical results. In particular, it delivers an appealing syntactic version of (iHI$^{\scriptscriptstyle \contract}_{\scriptscriptstyle \preccurlyeq}$) based on the rational closure $\CRat(\Gamma)$ of a set of conditionals $\Gamma$, or equivalently of a non-monotonic consequence relation (see  \cite{lehmann1992does}):

\begin{tabbing}
    \=BLAHBLAII\=\kill

\> (iHI$^{\scriptscriptstyle \contract}_{\scriptscriptstyle >}$) \> $\cbel{\Psi \contract A}=\CRat(\cbel{\Psi} \cap \cbel{\Psi \ast \neg A})$\\[-0.25em]
\end{tabbing} 
\vspace{-1em}

\noindent In other words, the conditional belief set obtained after contraction by $A$ corresponds to the rational closure of the intersection of the prior conditional belief set with the conditional belief set obtained after revision by $\neg A$. This principle is attractive, due to the fact that $\CRat(\Gamma)$ has been argued to correspond to the most conservative way of extending a consequence relation (equivalently: conditional belief set) to a {\em rational} consequence relation (equivalently: conditional belief set). (iHI), therefore,  parsimoniously fixes the issue noted above in relation to (NiHI), which sometimes resulted in a non-rational conditional belief set.  We return to TeamQueue aggregation below, in Section \ref{sec:TQAgg}.

\section{Background on parallel belief contraction} 
\label{sec:PrincParaBelCh}

While the ``serial'' model takes {\em single} sentences as inputs for contraction or revision, it has been suggested that this imposes unrealistic limitations on the kind of change that can be modelled. The problem of so-called  ``parallel'' (aka ``package'' or ``multiple'') contraction is to compute the  impact, on an agent's beliefs, of the simultaneous removal  of a non-empty finite indexed {\em set} $S= \{A_1,\ldots, A_n\}$ of sentences in $L$  (with set of indices $I=\{1,\ldots, n\}$). We shall denote parallel contraction  by $\odiv$ and assume that it subsumes $\contract$ as the special case in which the input is a singleton set, setting $\bel{\Psi \odiv \{A\}} = \bel{\Psi \contract A}$. We use $\bigwedge S$ to denote $A_1\wedge \ldots\wedge A_n$ and $\neg S$ to denote $\{\neg A\mid A\in S\}$. 

Considerations of parsimony motivate defining parallel contraction in terms of serial contraction. Regarding the single-step case, a number of the more straightforward proposals have been noted to be problematic.

First, we have the identification of parallel contraction by a set $S$ with a sequence of serial contractions by the members of $S$. This  runs into problems due to a failure of commutativity (see  \cite{HANSSONS.O1993RtLI}): different orders of operations can yield different outcomes, and no principled way seems to exist to privilege one order over another.

Second, there is the identification of parallel contraction by $S$ with a single serial contraction by some truth functional combination of the members of $S$ (such as the disjunction $\bigvee S$ of the members of $S$, so that, for example,  $\Psi\odiv \{A, B\}=\Psi\contract A\vee B$). Certainly, due to the logical  closure of belief sets, removing $A\vee B$ would involve removing both $A$ and $B$, as contraction by $\{A, B\}$ requires. However, as pointed out, for instance, in  \cite{FurHanSMC}, this would be too drastic an operation: clearly, one can simultaneously retract one's commitments both to $A$ and to $B$  without thereby retracting one's commitment to $A\vee B$. From this observation, it follows that we cannot generally identify the belief sets $\bel{\Psi\odiv\{A, B\}}$ and $\bel{\Psi\contract A\vee B}$ and hence a fortiori, that we cannot generally identify the belief {\em states} $\Psi\odiv \{A, B\}$ and $\Psi\contract A\vee B$. Furthermore, more generally, as Fuhrmann notes in \cite{Fuhrmann1996-FUHAEO},  there is no truth-functional combination of $A$ and $B$ that would do the job either. 

A more promising solution is the ``intersective'' approach, which identifies the belief set obtained by parallel contraction by $S$ with the intersection of the belief sets obtained by serial contraction by the members of $S$. This proposal has been endorsed by Spohn  \cite{SpohnPC}, as it follows from his more general approach to iterated parallel contraction.
\begin{tabbing}
    \=BLAHBLAII\=\kill

\> \MultiConInter \> $[\Psi\odiv\{A_1,\ldots, A_n\}]=\bigcap_{1\leq i\leq n}[\Psi\contract A_i]$\\[-0.25em]
\end{tabbing} 
\vspace{-1em}

\noindent This suggestion owes its plausibility to the same kind of considerations as the formally related Harper Identity did (see Subsection \ref{ss:SSSC}). In contracting by a set of sentences, the thought goes, we ought not believe anything that any of the individual contractions would preclude us from believing. But this  is the only modification to our prior beliefs that we should make. In particular, we should retract nothing further and introduce nothing new.

Beyond this rationale, we note that \MultiConInter~has several further attractive properties. First of all, if one assumes the basic AGM postulates for serial contraction, i.e. \KCon{1} to \KCon{6}, it yields a parallel contraction operator that satisfies plausible generalisations of  these:

\begin{restatable}{thm}{RedtoAGMPOneSix}
\label{thm:RedtoAGMPOneSix}
Let $\odiv$ be a parallel contraction operator such that, for some serial contraction operator $\contract$ that  satisfies \KCon{1}-\KCon{6}, $\odiv$ and  $\contract$ jointly satisfy  \MultiConInter. Then $\odiv$ satisfies:
\begin{tabbing}
    \=BLAHBLAII\=\kill

\> \KConP{1} \> $\Cn([\Psi\odiv  S ])\subseteq [\Psi\odiv  S ]$\\[0.1cm]

\> \KConP{2}  \> $[\Psi\odiv   S ]\subseteq[\Psi]$\\[0.1cm]

\> \KConP{3}  \> If $ S \cap[\Psi]=\varnothing$, then $[\Psi\odiv  S ]=[\Psi]$\\[0.1cm]

\> \KConP{4} \> $\forall A\in S$, if $  A\notin \Cn(\varnothing)$, then $ A\notin [\Psi\odiv S]$\\[0.1cm]

\> \KConP{5} \>  If $ S \subseteq[\Psi]$, then $[\Psi]\subseteq\Cn([\Psi\odiv  S ]\cup\{  S \})$\\[0.1cm]

\> \KConP{6} \> If, $\forall A_1\in S _1$, $\exists A_2\in S _2$ s.t. $\Cn(A_1)=\Cn(A_2)$,\\
\> \>  and vice versa, then $[\Psi\odiv  S _1]=[\Psi\odiv  S _2]$\\[-0.25em]

\end{tabbing} 
\vspace{-1em}
\end{restatable} 

\noindent These postulates were all endorsed by Fuhrmann \& Hansson (see \cite{FurHanSMC}) as plausible generalisations of their serial counterparts, with the exception of \KConP{4}. Their own generalisation of  \KCon{4} is indeed slightly weaker, although the difference only pertains to the handling of certain limiting cases.

It is worth noting that \KConP{6} differs from the following alternative generalisation of \KCon{6}: If $\Cn(S_1) = \Cn(S_2)$, then $[\Psi\odiv  S _1]=[\Psi\odiv  S _2]$. Indeed, as Fuhrmann \& Hansson  \cite[pp.~52]{FurHanSMC} point out, the latter is actually  {\em inconsistent} with the conjunction of \KConP{3} and \KConP{4}. To see why, where $p$ and $q$ are atomic sentences, let $\bel{\Psi} = \Cn(p)$. Then we have $\bel{\Psi\odiv\{p\wedge q\}} = \bel{\Psi}$, by \KConP{3}, since  $p\wedge q\notin \bel{\Psi}$, but  $\bel{\Psi\odiv\{p\wedge q, p\}} \neq \bel{\Psi}$, by \KConP{4}, even though $\Cn(\{p\wedge q\}) = \Cn(\{p\wedge q, p\})$.

Fuhrmann \& Hansson propose a pair of postulates constraining the relation between contractions by sets standing in a subset relation to one another. The first of these is satisfied by the intersective approach, while the second is not, even assuming \KCon{1}-\KCon{8}: 

\begin{restatable}{prop}{RedtoAGMPSevEightFH}
\label{prop:RedtoAGMPSevEightFH}
(a) Let $\odiv$ be a parallel contraction operator such that, for some serial contraction operator $\contract$, $\odiv$ and  $\contract$ jointly satisfy  \MultiConInter. Then $\odiv$ satisfies:

\begin{center}
If $S _1\cap  S _2\neq \varnothing$, then  $[\Psi\odiv  S _1]\cap[\Psi\odiv  S _2]\subseteq [\Psi\odiv  (S _1\cap  S _2)]$
\end{center}

\noindent (b) There exist a parallel contraction operator $\odiv$ and a serial contraction operator $\contract$ that jointly satisfy  \MultiConInter~but are such that the following principle fails:

\begin{center}
If $ S _1\cap[\Psi \odiv S _2]=\varnothing$, then $[\Psi \odiv  S _2]\subseteq [\Psi \odiv  (S _1\cup  S _2)]$
\end{center}

\end{restatable}

\noindent This is exactly as things should be, since, while the first postulate is plausible, the second seems quite dubious:

\begin{example} \label{ex:FandHKEight}
I initially believe that Alfred and Barry are both guilty ($A, B\in \bel{\Psi}$) but would entirely suspend judgment on the situation if I gave up on that belief (so that I would endorse no non-tautological combination of $A$ and $B$ in $\bel{\Psi\odiv \{A\wedge B\}}$). If I gave up my belief that Barry is guilty, I would no longer believe that Alfred is guilty ($A\notin\bel{\Psi\odiv \{B\}}$). However, while, in that situation, I would still believe that, if Barry is guilty, then Alfred is so too ($B\rightarrow A\in\bel{\Psi\odiv \{B\}}$), I would no longer believe this if I simultaneously gave up on both the belief that Barry is guilty and the belief that Alfred is ($B\rightarrow A\notin\bel{\Psi\odiv \{A, B\}}$).
\end{example}

\noindent If we take $S_1$ and $S_2$ to be respectively given by $\{A\}$ and $\{B\}$, this perfectly rationally acceptable situation runs contrary to the prescription made in the second postulate.

Although Fuhrmann \& Hansson propose  the above two principles  as  ``very tentative generalisations'' of \KCon{7} and  \KCon{8}, respectively, these cannot really be  ``generalisations'' in any obvious sense of the term, since the corresponding AGM postulates do not seem to be recoverable as special cases. This then raises the question of whether there exist any promising candidates for generalisations of \KCon{7} and  \KCon{8}. The following are obvious suggestions:

\begin{tabbing}
    \=BLAHBLAII\=\kill

\> \KConP{7} \> $[\Psi\odiv  S_1]\cap [\Psi\odiv  S_2]\subseteq [\Psi\odiv  \{\bigwedge (S_1\cup S_2)\}]$ \\[0.1cm]

\> \KConP{8}  \> If $S_1\cap\bel{\Psi\odiv\{\bigwedge (S_1\cup S_2)\}}=\varnothing$, then\\
\> \> $\bel{\Psi\odiv\{\bigwedge (S_1\cup S_2)\}}\subseteq \bel{\Psi\odiv S_1}$  \\[-0.25em]
\end{tabbing} 
\vspace{-1em}

\noindent Again, the intersective approach delivers here:

\begin{restatable}{thm}{RedtoAGMPSevenEight}
\label{thm:RedtoAGMPSevenEight}
Let $\odiv$ be a parallel contraction operator such that, for some serial contraction operator $\contract$, $\odiv$ and  $\contract$ jointly satisfy  \MultiConInter. Then, (i) if $\contract$ satisfies \KCon{7}, then  $\odiv$ satisfies \KConP{7} and, (ii) if $\contract$ satisfies \KCon{8}, then  $\odiv$ satisfies \KConP{8}.
\end{restatable} 

\noindent Despite its considerable appeal, \MultiConInter~has been explicitly rejected by Fuhrmann \& Hansson  (see \cite[pp.~51-57]{FurHanSMC}) due to its entailing the following  monotonicity principle: If $ S _1\subseteq S _2$, then $[\Psi \odiv  S _2]\subseteq[\Psi \odiv  S _1]$. This principle is not satisfied by their preferred constructive approach, which is governed by a remarkably weak set of principles that falls strictly short of  \KConP{1}-\KConP{8}. However, as Spohn has noted in  \cite{SpohnPC},  in the absence of a convincing, independent story as to {\em why} parallel contraction should be be governed by principles no stronger than the ones they endorse, this remains insufficient ground for criticism.

So much for single-step parallel contraction. What about the iterated case?  Surprisingly, next to no work has been carried out on this issue. Indeed,  to the best of our knowledge,   \cite{SpohnPC} is the only existing proposal regarding how this issue should be handled. However, although  Spohn's suggestion has the desirable feature of entailing  \MultiConInter, it relies heavily on his ranking theoretic formalism, the foundations of which still require a careful assessment \cite{ChandlerSpohn}. 

In what follows, we shall propose a more straightforward way of extending the intersective approach to the iterated case. Our key insight is that the situation here is analogous to the one faced in relation to extending the Harper Identity. In that situation, in the single-step case, we also had a proposal involving an intersection of belief sets (the sets $\bel{\Psi}$ and $\bel{\Psi\ast \neg A}$). In the iterated case, we faced the task of aggregating a number of conditional belief sets ($\cbel{\Psi}$ and $\cbel{\Psi\ast \neg A}$) or, equivalently, TPOs ($\preccurlyeq_\Psi$ and $\preccurlyeq_{\Psi\ast \neg A}$). The TPO aggregation procedure used in that context, however, was only characterised for {\em pairs} of TPOs. In the present context, we will need to aggregate {\em arbitrarily large finite sets} of TPOs.

\section{TeamQueue aggregation} 

\label{sec:TQAgg}

In this section, we offer generalisations to the $n$-ary case of the construction and characterisation results of the family of binary aggregators studied in \cite{DBLP:journals/ai/BoothC19} and two of their noteworthy special cases. 

The formal framework involves the following: a finite set of alternatives $W$, a finite non-empty set of indices $I =\{1,\ldots,n\}$, a tuple $\mathbf{P} = \langle\preccurlyeq_{i}\rangle_{i\in I}$ of  TPOs over $W$ known as ``profiles''  and an  aggregation function $\oplus$ mapping  all possible profiles onto single TPOs over $W$. When we shall need to refer to multiple profiles and their constituent relations, we shall use superscripted roman numerals, writing $I^j =\{1,\ldots,n^j\}$ and  $\mathbf{P}^j = \langle\preccurlyeq^j_{i}\rangle_{i\in I^j}$. When the identity of $\mathbf{P} $ is clear from context, we shall write $\oplus$ to denote  $\oplus (\mathbf{P})$ and  $x \preccurlyeq_{\oplus} y$ to denote $\langle x, y\rangle\in \oplus$, or simply $x \preccurlyeq_{\oplus} y$.

As was mentioned above, TQ aggregation was originally introduced after observing the fact that a particular identity involving the intersection of two conditional belief sets (namely (NiHI$^{\scriptscriptstyle \contract}_{\scriptscriptstyle >}$)) clashed with the AGM postulates. This result is in fact related to a more general observation, made in \cite{lehmann1992does}, that the intersection of two sets of rational conditionals needn't itself be rational. In other words, the following naive  principles of ``Conditional Intersection'' make poor suggestions, if we require that $\cbel{\preccurlyeq_{\oplus}}$ be rational or $\preccurlyeq_{\oplus}$ be a TPO:

\begin{tabbing}
    \=BLAHBLAII\=\kill
\> (CI$^{\scriptscriptstyle \oplus}_{\scriptscriptstyle >}$) \> $\cbel{\preccurlyeq_{\oplus}}=\bigcap_{i\in I}\cbel{\preccurlyeq_i}$\\[0.1cm]
\> (CI$^{\scriptscriptstyle \oplus}_{\scriptscriptstyle \min}$) \> For all $S\subseteq W$, $\min(\preccurlyeq_{\oplus},  S)$\\
\>\>$= \bigcup_{i\in I} \min(\preccurlyeq_i,  S )$\\[-0.25em]
\end{tabbing} 
\vspace{-1em}

\noindent The following example makes the point:

\begin{example}
\label{eg:Naivefails}

Let $W=\{x, y, z, w\}$ and $\preccurlyeq_1$ and $\preccurlyeq_2$ be respectively given by: $x  \prec_1 \{w, z\} \prec_1 y$
and $z \prec_2 y \prec_2 x \prec_2 w$.

(CI$^{\scriptscriptstyle \oplus}_{\scriptscriptstyle \min}$) has the consequence that $\preccurlyeq_{\oplus}$ isn't a TPO. Indeed, $\min(\preccurlyeq_1, W)\cup \min(\preccurlyeq_2, W) = \{x, z\}$. So by (CI$^{\scriptscriptstyle \oplus}_{\scriptscriptstyle \min}$), we have $z\preccurlyeq_\oplus x$ and $x\prec_\oplus w$. On the assumption that  $\preccurlyeq_{\oplus}$  is a TPO, this gives us $z\prec_\oplus w$. However, from the fact that $\min(\preccurlyeq_1, \{w, z\})\cup \min(\preccurlyeq_2, \{w, z\}) = \{w, z\}$, we have, by (CI$^{\scriptscriptstyle \oplus}_{\scriptscriptstyle \min}$), $w\preccurlyeq_\oplus z$. Contradiction.

Similarly, (CI$^{\scriptscriptstyle \oplus}_{\scriptscriptstyle >}$) has the consequence that  $\cbel{\preccurlyeq_{\oplus}}$ isn't rational. First, from the fact that  $(x\vee w) > \neg w \in \cbel{\preccurlyeq_1}\cap \cbel{\preccurlyeq_2}$, by (CI$^{\scriptscriptstyle \oplus}_{\scriptscriptstyle >}$), we have: (i) $(x\vee w) > \neg w \in \cbel{\preccurlyeq_\oplus}$. Second, from the fact that $(x\vee z) > \neg z  \notin\cbel{\preccurlyeq_2}$, by (CI$^{\scriptscriptstyle \oplus}_{\scriptscriptstyle >}$) it follows that :(ii)  $(x\vee z) > \neg z \notin \cbel{\preccurlyeq_\oplus}$. Finally, from $(w\vee z) > \neg w  \notin \cbel{\preccurlyeq_1}$, by the same principle again: (iii) $(w\vee z) > \neg w \notin \cbel{\preccurlyeq_\oplus}$. However, taken together, (i)--(iii) directly violate a principle that is valid for rational conditionals (see \cite[Lem. 17]{lehmann1992does}).

\end{example}

\subsection{Construction} 

\noindent The $n$-ary version of the aggregation method is constructively defined in a very similar way to that in which the original binary case was. The definition makes use of the representation of a TPO $\preccurlyeq$ by means of an ordered partition $\langle S_1, S_2, \ldots S_{m_i}\rangle$ of $W$, defined inductively as follows: $S_1 = \min(\preccurlyeq, W)$ and, for $i\geq 2$, $S_i=\min(\preccurlyeq, \bigcap_{j<i} S^c_j)$, where $S^c$ is the complement of $S$. This representation grounds the notion of  the {\em absolute} rank $r(x)$ of an alternative $x$, with respect to $\preccurlyeq$. The absolute rank of an alternative is given by its position in the ordered partition, so that $r(x)$ is such that $x\in S_{r(x)}$ (in cases in which the TPO is indexed, as in $\preccurlyeq_i$, we write $r_i(x)$). With this in mind, we can offer:

\begin{definition}
$\oplus$ is a {\em TeamQueue (TQ) aggregator} iff, for each profile $\mathbf{P}$ with index set $I$, there exists a sequence $\langle a_{\mathbf{P}}(i)\rangle_{i \in \mathbb{N}}$ such that $\emptyset \neq a_{\mathbf{P}}(i) \subseteq I$ for each $i$  and the ordered partition $\langle T_1, T_2, \ldots, T_m\rangle$ of indifferences classes corresponding to $\preccurlyeq_{\oplus}$ is constructed inductively as follows:
\[
T_{i} = \bigcup_{j \in a_{\mathbf{P}}(i)} \min(\preccurlyeq_j, \bigcap_{k<i}T_{k}^c)
\]
where  $m$ is minimal s.t. $\bigcup_{i\leq m} T_i = W$. 
\end{definition}

\noindent The procedure takes the input TPOs and processes them step by step to form a new TPO. At the first step, it removes the minimal elements of one or more of the input TPOs (which TPOs these are depends on the specifics of the procedure, i.e. on the value(s) in $a_{\mathbf{P}}(i)$ for the relevant step $i$) and places them in the minimal rank of the output TPO, before deleting any copies of these elements that might remain in the input TPOs. At each step, it then repeats the process using the remainders of the input TPOs, until all input TPOs have been processed entirely.\footnote{In  \cite{DBLP:journals/ai/BoothC19}, which simply discussed the binary case, the additional requirement that $a_{\mathbf{P}}(1) = \{1, 2\}$ was imposed. We do not require the $n$-ary generalisation of this requirement (i.e. $a_{\mathbf{P}}(1) = \{1,\ldots, n\}$).
}

Of particular interest is the member of the TQ aggregator family that processes the TPOs ``synchronously'', so that, at each step, the minimal elements of all TPOs are included in the relevant output rank:

\begin{definition}
The {\em Synchronous TeamQueue (STQ)} aggregator $\STQ$ is the TeamQueue aggregator for which $a_{\mathbf{P}}(i) = \{1,\ldots, n\}$ for all profiles $\mathbf{P}=\langle \preccurlyeq_1,\ldots,  \preccurlyeq_n \rangle$ and all $i$. 
\end{definition}

\noindent Another noteworthy TQ aggregator is the ``MinRank'' aggregator, whose binary version is briefly discussed in footnote 11 of \cite{DBLP:journals/ai/BoothC19} and shown there to be distinct from $\STQ$: 

\begin{definition}
The {\em MinRank} aggregator $\oplus_{\min}$ is the aggregator s.t. $x\preccurlyeq_{\oplus} y$ iff $\arg\min_{i\in I}r_i(x) \leq \arg\min_{i\in I}r_i(y) $.  
\end{definition}

\noindent While $\oplus_{\min}$ assigns to $x$ the minimal rank {\em it received among  the inputs}, $\STQ$ assigns to $x$ the minimal rank {\em it can receive within the constraints imposed by TQ aggregation}. The following example illustrates the way in which these aggregators can yield different outputs:

\begin{example}
\label{eg:STQvsMinRank}
Let $\mathbf{P}= \langle \preccurlyeq_1,\ldots, \preccurlyeq_4\rangle$, where: $\{w\} \prec_1 \{z\} \prec_1 \{x, y\}$, $\{w\} \prec_2 \{y\} \prec_2 \{x, z\}$, $\{z\} \prec_3 \{w\} \prec_3 \{x, y\}$, and $\{z\} \prec_4 \{y\} \prec_3 \{x, w\}$. We have $\STQ$ given by $\{w, z\} \prec_{\STQ}  \{x, y\}$ but $\oplus_{\min}$ is given by $\{w, z\} \prec_{\oplus_{\min}}  \{y\} \prec_{\oplus_{\min}}  \{x\}$. See Figure \ref{fig:Aggreg}.
\end{example}

\begin{figure}
  \centering
  \vspace{-1em}
    \includegraphics[width=0.4\textwidth]{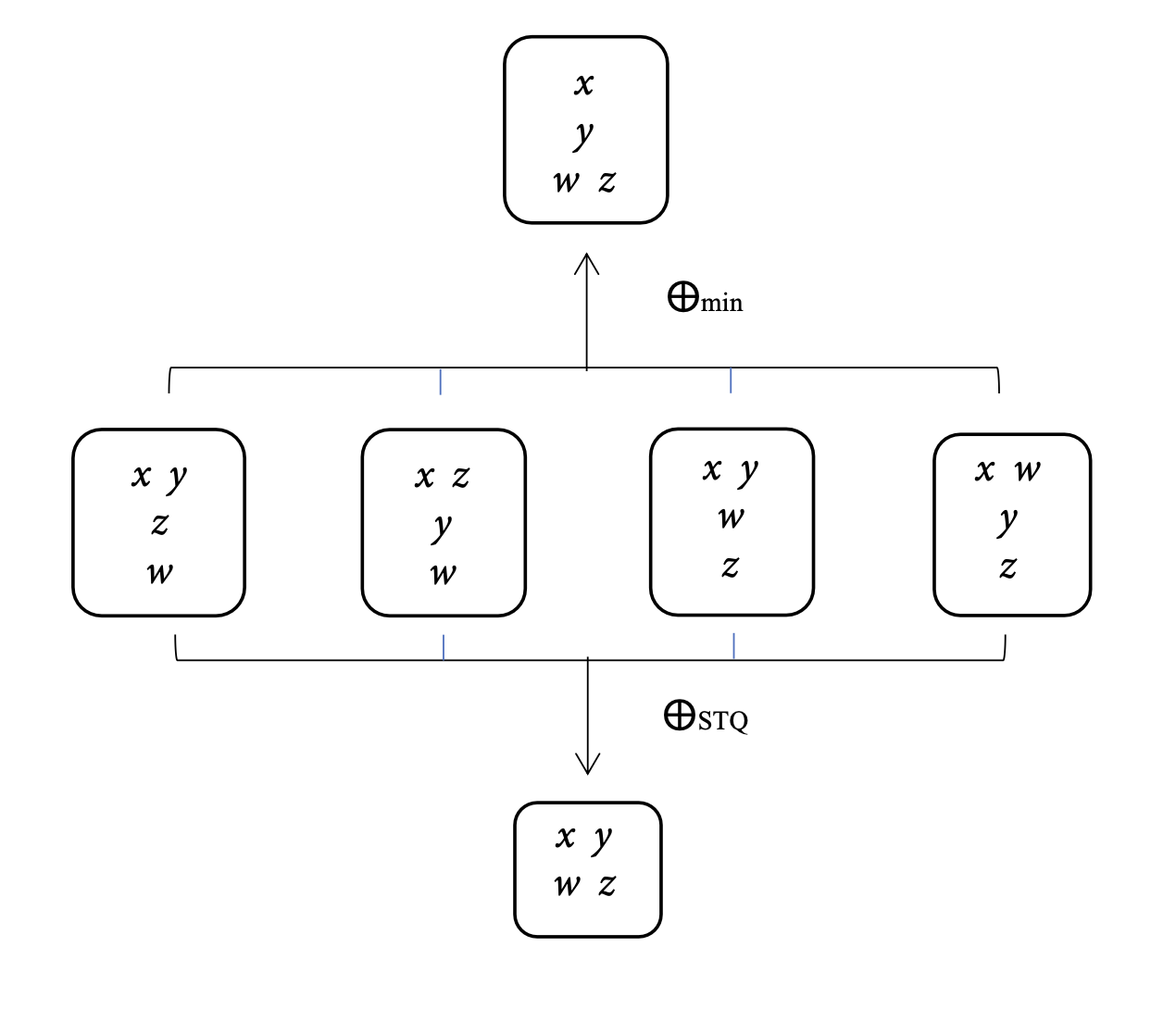}
      \vspace{-1em}
  \caption{Illustration of the $\STQ$ and $\oplus_{\min}$ aggregations in Example \ref{eg:STQvsMinRank}. Boxes represent TPOs, with lower case letters arranged such that a lower letter corresponds to a lower world in the relevant ordering.}
  \label{fig:Aggreg}
\end{figure}

\subsection{Characterisation} 
\label{subsec:charact}

The family of $n$-ary TeamQueue aggregators can be characterised in terms of minimal sets as follows:

\begin{restatable}{thm}{NaryRep}
\label{thm:NaryRep}

$\oplus$ is a TeamQueue aggregator iff it satisfies the following ``factoring'' property:
\begin{tabbing}
    \=BLAHBLAII\=\kill
\> \FacMin \> For all $S\subseteq W$, there exists $X\subseteq I$, s.t.\\
\> \> $\min(\preccurlyeq_{\oplus},  S )$$= \bigcup_{j\in X} \min(\preccurlyeq_j,  S )$\\[-0.25em]
\end{tabbing} 
\vspace{-1em}

\end{restatable}

\noindent This is a weakening of the principle (CI$^{\scriptscriptstyle \oplus}_{\scriptscriptstyle \min}$) discussed above, which we have seen cannot hold in full.
 
 The characterisation can also be given in terms of a property requiring that no element of $W$ can improve its relative position with respect to {\em all} input orderings:

\begin{restatable}{prop}{FSyntSem}
\label{prop:FSyntSem}

\FacMin~is equivalent to 

\begin{tabbing}
\=BLAHBLAI\= ooooo \= \kill
\> \FacPref \> Assume that $x_1$ to $x_n$ are s.t.~$x_i \preccurlyeq_i y$. Then  there\\
\> \>   exists  $j\in I$ s.t. \>  \\[0.1cm]
\>  \>  (i)  \> if  $x_j \prec_j y$, then  $x_j \prec_\oplus y$, and \\[0.1cm]
\> \>  (ii)  \> if  $x_j \preccurlyeq_j y$, then  $x_j \preccurlyeq_\oplus y$\\[-0.25em]
\end{tabbing} 

\end{restatable} 

\vspace{-1em}

\noindent The characterisation of $\STQ$ can be achieved by supplementing  \FacPref~ with a principle of ``Parity'':

\begin{restatable}{thm}{PAR}
\label{thm:PAR}
$\STQ$ is the only aggregator that satisfies both \FacPref~and the following `Parity' constraint:
\begin{tabbing}
    \=BLAHBLAII\=\kill
\> \PPAR   \>  If $x \prec_{\oplus} y$ then, for each $i \in I$, there exists $z_i$ s.t.\\
\> \> $x \sim_{\oplus} z_i$ and  $z_i \prec_i y$\\[-0.25em]
\end{tabbing} 
\vspace{-1em}
\end{restatable}

\noindent This principle can also be framed in terms of minimal sets:

\begin{restatable}{prop}{PARSB}
\label{prop:PARSB}
\PPAR~is equivalent to: 
\begin{tabbing}
    \=BLAHBLAII\=\kill
\> \PPARMin   \>  If $x\prec_\oplus y $ for all $x \in S^c$, $y \in S$, then\\
\> \> $\bigcup_{i\in I}\min( \preccurlyeq_{i},S)\subseteq \min( \preccurlyeq_{\oplus},S)$\\[-0.25em]
\end{tabbing}
 \end{restatable}
 \vspace{-1em}

\noindent The possibility of characterising $\oplus_{\min}$ in similar terms remains an open question.

\subsection{Further properties} 
\label{subsec:furtherprops}

Like its binary special case, $n$-ary TeamQueue aggregation satisfies several Pareto-style properties. In particular, we note that:

\begin{restatable}{prop}{SemBreakdown}
\label{prop:SemBreakdown} 
\FacPref~entails the following two properties:
\begin{tabbing}
    \=BLAHBLAII\=\kill
\>   \SPUPlus  \> Assume that, for all $i\in I$, $x_i \prec_i y$. Then  there\\
\> \>   exists $j\in I$ s.t.~$x_j \prec_\oplus y$  \\[0.1cm]
\>  \WPUPlus   \> Assume that, for all $i\in I$, $x_i \preccurlyeq_i y$. Then  there\\
\> \>   exists   $j\in I$ s.t.~$x_j \preccurlyeq_\oplus y$  \\[-0.25em]
\end{tabbing}
\vspace{-1em}
\end{restatable}

\noindent These properties  \SPUPlus~and  \WPUPlus~can be equivalently be framed in terms of minimal sets, as follows:

\begin{restatable}{prop}{SWPUSynt}
\label{prop:SWPUSynt}
 \SPUPlus~and  \WPUPlus~are respectively equivalent to:

\begin{tabbing}
    \=BLAHBLAII\=\kill
\> \UBO \> For all $S\subseteq W$, $\min(\preccurlyeq_\oplus, S) \subseteq \bigcup_{i\in I} \min(\preccurlyeq_i, S)$ \\[0.1cm] 
\> \LBO   \> For all $S\subseteq W$, there exists $i\in I$  s.t.~$\min(\preccurlyeq_i, S)$\\
\> \> $\subseteq \min(\preccurlyeq_\oplus, S)  $ \\[-0.25em]
\end{tabbing} 
\vspace{-1em}

\end{restatable}

\noindent  \SPUPlus~and  \WPUPlus~generalise the well-known Social Choice  properties of Weak Pareto and Pareto Weak Preference, which can be respectively given as: 

\begin{tabbing}
    \=BLAHBLAII\=\kill
\>   \SPU   \>  If,  for all $i\in I$, $x\prec_i  y$, then  $x\prec_\oplus y$  \\[0.1cm]
\>  \WPU   \>  If,  for all $i\in I$, $x\preccurlyeq_i  y$, then  $x\preccurlyeq_\oplus y$  \\[-0.25em]
\end{tabbing}
\vspace{-1em}

\noindent \SPU~and   \WPU~can also be formulated in terms of upper and lower bounds on the output relation $\preccurlyeq_\oplus$, jointly corresponding to $\bigcap_{i\in I}\preccurlyeq_i ~ \subseteq ~ \preccurlyeq_\oplus ~ \subseteq ~\bigcup_{i\in I}\preccurlyeq_i$.

\subsection{The connection to rational closure}

 \noindent  As mentioned above, in Section \ref{ss:ItSer}, a connection was drawn in \cite{DBLP:journals/ai/BoothC19} between the binary special case of $\STQ$ and the concept of the rational closure of a set of conditionals . 
 It was shown, as a corollary of a key theorem, that the conditional belief set corresponding to the TPO obtained by aggregation of two input TPOs is given by the rational closure of the conditional belief sets corresponding to these input TPOs (see their Theorem 3 and Corollary 1). Here we can report that this theorem and its corollary generalise straightforwardly to the $n$-ary case. Indeed, we first recall the following definition:

\begin{definition}
\label{dfn:Flatter}
Let $\sqsupseteq$ be a binary relation on the set of TPOs over $W$ s.t. 
$
\langle S_1, S_2, \ldots, S_m \rangle
\sqsupseteq$$
\langle T_1, T_2, \ldots, T_m \rangle 
$
iff
either (i) $S_i = T_i$ for all $i = 1, \ldots, m$, or (ii) $S_i \supset T_i$ for the first $i$ s.t. $S_i \neq T_i$. 
\end{definition}

\noindent The relation $\sqsupseteq$ intuitively partially orders TPOs by what one could call comparative ``flatness''. So, for instance, where $\preccurlyeq_1$ and $\preccurlyeq_2$ are respectively given by $\{x, y\} \prec_1 \{z, w\}$ and $\{x, y\} \prec_2 z \prec_2 w$ and so $\preccurlyeq_1$ is intuitively ``flatter'' than $\preccurlyeq_2$, we have $\preccurlyeq_1\,\sqsupseteq\,\preccurlyeq_2$. We can then show the following: 

\begin{restatable}{thm}{Flattest}
\label{thm:Flattest}
$\preccurlyeq_{\STQ} \, \sqsupseteq \, \preccurlyeq_{\oplus}$ for  any aggregator $\oplus$ satisfying \SPUPlus.
\end{restatable}

\noindent To appreciate the significance of this result, we firstly need to understand how \SPUPlus~translates into the language of conditionals. Recall that this property was shown to be equivalent to a property that we called \UBO. This second property can be presented in terms of sets of conditionals as stating that the intersection of the sets of conditionals corresponding to the inputs is included in the set of conditionals corresponding to the output: 

\begin{tabbing}
    \=BLAHBLAII\=\kill
\> (UB$^{\scriptscriptstyle \oplus}_{\scriptscriptstyle >}$) \> $\bigcap\cbel{\preccurlyeq_i}\subseteq \cbel{\preccurlyeq_{\oplus}}$ \\[-0.25em]
\end{tabbing} 

\vspace{-1em}

\noindent Theorem \ref{thm:Flattest}, then, tells us that $\STQ$ returns the flattest TPO whose corresponding conditional belief set contains the intersection of the conditional belief sets corresponding to the input TPOs. Secondly, we know from Booth \& Paris \cite{Booth1998-BOOANO}  that the rational closure of a set of conditionals corresponds to the flattest TPO that satisfies it. Finally, putting the above two observations together then leaves us with the following immediate corollary: 

\begin{restatable}{cor}{Rat}
\label{cor:Rat}
$\cbel{\preccurlyeq_{ \STQ} }=\CRat(\bigcap_i \cbel{\preccurlyeq_i})$
\end{restatable}

\section{Parallel contraction via TeamQueue aggregation} 

\label{sec:TQPara}

An obvious suggestion is to define iterated parallel contraction in terms of iterated contraction, using TeamQueue aggregation, as follows: 
\begin{tabbing}
    \=BLAHBLAII\=\kill
\> \MultiConAggregPrec\> $\preccurlyeq_{\Psi\odiv\{A_1,\ldots, A_n\}}=\oplus\{\preccurlyeq_{\Psi\contract A_1}, \ldots, \preccurlyeq_{\Psi\contract A_n}\}$\\[-0.25em]
\end{tabbing} 
\vspace{-1em}

\noindent If we impose the constraint that $a_{\mathbf{P}}(1) = \{1,\ldots, n\}$ on the construction of $\oplus$, as is the case in relation to $\STQ$, then this suggestion yields \MultiConInter--the principle according to which the belief set obtained after contraction by a set $S$ is given by the intersection of the belief sets obtained after contractions by each of the members of $S$--as its special case for single-step contraction. 

We can then use the above principle to define the class of TeamQueue parallel contraction operators: 

\begin{definition}
\label{def:TQPackageContraction}
$\odiv$ is a {\em TeamQueue parallel contraction operator} if and only if there exists an AGM contraction operator $\contract$, s.t. $\odiv$ and $\contract$ jointly satisfy \MultiConAggregPrec, where $\oplus$ is a TeamQueue aggregator.
\end{definition}

\noindent More specific concepts, such as, for example, that of an STQ parallel contraction operator, can be defined in the same manner. As an immediate corollary of Theorem \ref{thm:NaryRep}, we then also have the following characterisation result:

\begin{restatable}{cor}{TQPackageContractionCharacterisation}
\label{cor:TQPackageContractionCharacterisation}
$\odiv$ is a TeamQueue parallel contraction operator if and only if it satisfies 

\begin{tabbing}
    \=BLAHBLAII\=\kill
\> \FacConBel \> For all $B\in L$, there exists $X\subseteq I$ s.t. \\
\> \> $\bel{(\Psi\odiv S) \ast B}= \bigcap_{i\in X} \bel{(\Psi\contract A_i)\ast B}$\\[-0.25em]
\end{tabbing} 
\vspace{-1em}

\end{restatable}

\noindent The various results of sections \ref{subsec:charact} and  \ref{subsec:furtherprops} also have straightforward corollaries, starting with the following immediate joint consequence of  Propositions \ref{prop:SemBreakdown} and \ref{prop:SWPUSynt}:

\begin{restatable}{cor}{TQPackageContractionSound}
\label{cor:TQPackageContractionSound}
If $\odiv$ is a TeamQueue parallel contraction operator then it satisfies 

\begin{tabbing}
    \=BLAHBLAII\=\kill
\> \UBOConBel \> For all $B\in L$, $\bigcap_{i\in I} \bel{(\Psi\contract A_i)\ast B}\subseteq$\\
\> \> $  \bel{(\Psi\odiv S) \ast B }$ \\[0.1cm] 
\> \LBConBel   \> For all $B\in L$,  $\bel{(\Psi\odiv S) \ast B}\subseteq$\\
\> \> $ \bigcup_{i\in I} \bel{(\Psi\contract A_i)\ast B}$ \\[-0.25em]
\end{tabbing} 

\end{restatable}

\vspace{-1em}

\noindent Importantly, TeamQueue parallel contraction operators satisfy some rather compelling analogues of \CConR{1}--\CConR{4} for the parallel case:

\begin{restatable}{prop}{CtoCPack}
\label{prop:CtoCPackTwo}
Let $\odiv$ be a parallel contraction operator such that, for some AGM contraction operator $\contract$ and  TeamQueue aggregator $\oplus$, $\odiv$, $\contract$ and $\oplus$ jointly satisfy \MultiConAggregPrec. Then, if $\contract$ satisfies \CConR{1}--\CConR{4},  then $\odiv$ satisfies:

\begin{tabbing}
    \=BLAHBLAII\=\kill

\> \CConRn{1}  \>  If $x,y \in\mods{\bigwedge \neg S}$ then $x \preccurlyeq_{\Psi \odiv S}   y$ iff  $x \preccurlyeq_\Psi y$ \\[0.1cm]

\> \CConRn{2} \> If $x,y \in\mods{\bigwedge S}$ then $x \preccurlyeq_{\Psi \odiv S}   y$ iff  $x \preccurlyeq_\Psi y$\\[0.1cm]

\> \CConRn{3}   \> If $x \in\mods{\bigwedge \neg S}$, $y \notin\mods{\bigwedge \neg S}$  and $x \prec_\Psi y$ then \\
\> \> $x \prec_{\Psi \odiv S}   y$ \\[0.1cm]

\> \CConRn{4} \> If $x \in\mods{\bigwedge \neg S}$, $y \notin\mods{\bigwedge \neg S}$ and $x \preccurlyeq_\Psi y$ then\\
\> \> $x \preccurlyeq_{\Psi \odiv S}   y$\\[-0.25em]
\end{tabbing} 
\vspace{-1em}

\end{restatable}

\noindent As a corollary of Theorem \ref{thm:PAR} we can provide a result, pertaining to STQ parallel contraction,  framed in terms of the concept of ``strong belief'', discussed in \cite{battigalli2002strong,stalnaker1996knowledge}:

\begin{definition}
$A$ is {\em strongly believed (s-believed)} in $\Psi$ iff (i) $A \in \bel{\Psi}$, and (ii) $A \in \bel{\Psi \ast B}$ for all sentences $B$ s.t. $A \wedge B$ is consistent.
\end{definition}

\noindent This result is the following:

\begin{restatable}{cor}{STQPackageContractionSound}
\label{cor:STQPackageContractionSound}
$\odiv$ is an STQ parallel contraction operator iff it is a TeamQueue parallel contraction operator that also satisfies: 

\begin{tabbing}
    \=BLAHBLAII\=\kill
\> \SBConBel   \>  If  $\neg B$ is s-believed in $\Psi\odiv S$, then \\
\> \> $ \bel{(\Psi\odiv S) \ast B} \subseteq \bigcap_{i\in I} \bel{(\Psi\contract A_i)\ast B}$ \\[-0.25em]
\end{tabbing}
\vspace{-1em}

\end{restatable}

\noindent Last but not least, Corollary \ref{cor:Rat}  translates into the following, connecting STQ iterated parallel contraction with rational closure:

\begin{restatable}{cor}{RatCon}
\label{cor:RatCon}
$\odiv$ is an STQ parallel contraction operator iff the following equality holds:  $\cbel{\Psi\odiv\{A_1,\ldots, A_n\}}= \CRat(\bigcap_i \cbel{\Psi\contract A_i})$
\end{restatable}

\section{Concluding comments}
\label{sec:Concl}

In this paper, we have proposed an original approach to the neglected issue of parallel belief contraction, based on the generalisation of a largely unexplored family of methods for TPO aggregation.

The method generalises to the iterated case the ``intersective'' approach to single-step parallel contraction, which we have demonstrated can derive (i) Furhmann and Hansson's parallel versions of the basic AGM postulates for serial contraction and (ii) a pair of new plausible generalisations of the relevant supplementary postulates.

While explicitly regulating {\em two}-step parallel change, the approach allows handling indefinitely many parallel contractions when used with serial contraction operators that identify epistemic states with TPOs, such as moderate or priority contraction operators.  For models using {\em richer} structures than TPOs, such as ordinal intervals \cite{DBLP:journals/ai/BoothC20} or ranking functions (see \cite{Spohn2009-SPOASO} for an overview, though note that Spohn's proposal does not involve aggregation), a parallel suggestion would require a suitable  adaptation of the aggregation method.

Looking beyond contraction, it would be valuable to investigate whether the TQ approach could be applied to iterated parallel {\em revision}. This topic remains under-explored, with the only significant work being \cite{10.1007/978-3-540-24609-1_27} and \cite{DelgrandeJames2012PbrR} (\cite{resinamultiplesurvey} survey work on the single-step case).

Finally, there may be applications of TQ aggregation beyond belief revision, as the aggregation of orderings appears in multiple areas. One might consider whether TQ aggregation could show promise in preference or judgment aggregation, as a method for aggregating conditional judgments, preference aggregation, or judgments regarding comparative magnitudes. However, TQ aggregation as presented here would need generalisation for such tasks, as it is currently insensitive to TPO duplication in the profile, meaning profiles with identical members yield the same output. While this property suits parallel iterated belief change, it may not suit these other domains.

\bibliographystyle{splncs04}
\bibliography{CHANDLER_biblio}

%=====================================================

\section*{Appendix: proofs}\label{s:appendix}

%=====================================================
\RedtoAGMPOneSix*

\begin{pproof}
Regarding \KConP{1}: By \MultiConInter~and \KCon{1}, we know that $[\Psi\odiv  S ]$ is the intersection of a set of deductively closed sets, which is well-known to itself be deductively closed.

Regarding \KConP{2}: Assume  $C\in[\Psi\odiv  S ]$. We need to show that $C\in [\Psi]$. By \MultiConInter,  $C\in\bigcap_{1\leq i\leq n}[\Psi\contract A_i]$. By \KCon{2}, we then recover $C\in [\Psi]$.

Regarding \KConP{3}: Assume $ S \cap[\Psi]=\varnothing$. By \MultiConInter, we simply need to show that $\bigcap_{1\leq i\leq n}[\Psi\contract A_i]=[\Psi]$. From $ S \cap[\Psi]=\varnothing$, we have, for all $i$ such that $1\leq i\leq n$, $A_i\notin \bel{\Psi}$. By \KCon{3}, it then follows that, for all $i$ such that $1\leq i\leq n$, $[\Psi\contract A_i]=[\Psi]$. Hence $\bigcap_{1\leq i\leq n}[\Psi\contract A_i]=[\Psi]$, as required.

Regarding \KConP{4}: Assume that $A_i\in S- \Cn(\varnothing)$. By \MultiConInter, we need to show that $A_i\notin\bigcap_{1\leq i\leq n}[\Psi\contract A_i]$. This immediately follows from \KCon{4}, which ensures that $A_i\notin[\Psi\contract A_i]$.

Regarding  \KConP{5}: Assume $ S \subseteq[\Psi]$. By \MultiConInter, we need to show that $[\Psi]\subseteq\Cn(\bigcap_{1\leq i\leq n}[\Psi\contract A_i]\cup\{  S \})$. By \KCon{5}, we have, for all $i$ such that $1\leq i\leq n$,   $[\Psi]\subseteq\Cn([\Psi\contract A_i]\cup\{ A_i\})$.  Let $B\in L$ be such that $\bel{\Psi} = \Cn(B)$ (by the finiteness of $L$, we know that such $B$ exists). By the Deduction Theorem for $\Cn$, for all $i$ such that $1\leq i\leq n$,   we have $A_i\rightarrow B \in [\Psi\contract A_i]$. By \KCon{1},  for all $i$ such that $1\leq i\leq n$,   we then have $\bigwedge S\rightarrow B \in [\Psi\contract A_i]$ and hence $\bigwedge S\rightarrow B \in \bigcap_{1\leq i\leq n}[\Psi\contract A_i]$. It then follows from this that $B \in \Cn(\bigcap_{1\leq i\leq n}[\Psi\contract A_i]\cup\{  S \})$ and hence that $[\Psi]\subseteq\Cn(\bigcap_{1\leq i\leq n}[\Psi\contract A_i]\cup\{  S \})$, as required. 

Regarding \KConP{6}: Assume that $\forall A^1\in S ^1$, $\exists A^2\in S ^2$ such that $\Cn(A^1)=\Cn(A^2)$, and vice versa.  By \MultiConInter, we need to show that $\bigcap_{1\leq i\leq n}[\Psi\contract A^1_i]=\bigcap_{1\leq i\leq m}[\Psi\contract A^2_i]$. By \KCon{6}, we know that,  $\forall A^1\in S ^1$, $\exists A^2\in S ^2$ such that   $[\Psi\contract A^1]=[\Psi\contract A^2]$. From this, it follows that $\bigcap_{1\leq i\leq n}[\Psi\contract A^1_i]\supseteq\bigcap_{1\leq i\leq m}[\Psi\contract A^2_i]$. But by the same postulate, we also know that $\forall A^2\in S ^2$, $\exists A^1\in S ^1$ such that   $[\Psi\contract A^2]=[\Psi\contract A^1]$. The converse inclusion then holds and we recover the required result. 
\end{pproof}

\vspace{1em}

%=====================================================

\RedtoAGMPSevEightFH*

\begin{pproof}
Regarding (a): On the one hand, \MultiConInter~entails the following Monotonicity principle:  If $ S _1\subseteq S _2$, then $[\Psi \odiv  S _2]\subseteq[\Psi \odiv  S _1]$. It follows from this that  $[\Psi \odiv  (S _1\cup  S _2)]\subseteq[\Psi\odiv  (S _1\cap  S _2)]$. On the other hand, \MultiConInter~also entails $[\Psi\odiv  S _1]\cap[\Psi\odiv  S _2]\subseteq [\Psi\odiv  (S _1\cup  S _2)]$. From these two implications, we recover $[\Psi\odiv  S _1]\cap[\Psi\odiv  S _2]\subseteq [\Psi\odiv  (S _1\cap  S _2)]$, as required.

Regarding (b): This countermodel has the structure of the situation depicted in Example \ref{ex:FandHKEight}. Let $W= \{x, y, z, w\}$,  $\mods{A\wedge B} = x$, $\mods{A\wedge \neg B} = y$, $\mods{\neg A\wedge B} = z$, and $\mods{\neg A\wedge \neg B} = w$.  Let $S_1 = \{A\}$ and $S_2=\{B\}$. Assume that $\contract$ satisfies \KCon{1}-\KCon{8} and let the TPO $\preccurlyeq_\Psi$ associated with $\Psi$ be given by $x\prec_{\Psi}\{y, z, w\}$. Then, $A\notin\bel{\Psi\odiv \{B\}}$, $B\rightarrow A\in\bel{\Psi\odiv \{B\}}$ but $B\rightarrow A\notin\bel{\Psi\odiv \{A, B\}}$ 
\end{pproof}

\vspace{1em}

%=====================================================

\RedtoAGMPSevenEight* 

\begin{pproof}
Regarding (i): By \MultiConInter, we know that $\bel{\Psi\odiv S_1}\cap\bel{\Psi\odiv S_1} = \bel{\Psi\odiv (S_1\cup S_2)}$. It follows that establishing  \KConP{7} is equivalent to showing $\bel{\Psi\odiv S} \subseteq \bel{\Psi\odiv \{\bigwedge S\}} $. By \MultiConInter, this is equivalent to $\bigcap_{A\in S}\bel{\Psi\contract A} \subseteq \bel{\Psi\contract \bigwedge S} $. But this follows by repeated applications of \KCon{7}.

Regarding (ii): Suppose  $S_1\cap\bel{\Psi\odiv\{\bigwedge (S_1\cup S_2)\}}=\varnothing$, i.e.~$S_1\cap\bel{\Psi\contract\bigwedge (S_1\cup S_2)}=\varnothing$. We must show $\bel{\Psi\odiv\{\bigwedge (S_1\cup S_2)\}}\subseteq \bel{\Psi\odiv S_1}$, i.e. by \MultiConInter,  $\bel{\Psi\contract\bigwedge (S_1\cup S_2)}\subseteq \bel{\Psi\contract A}$ for all $A\in S_1$. So let $A\in S_1$. Since $S_1\cap\bel{\Psi\contract\bigwedge (S_1\cup S_2)}=\varnothing$, we have $A\notin \bel{\Psi\contract\bigwedge (S_1\cup S_2)}$. We then recover $\bel{\Psi\contract\bigwedge (S_1\cup S_2)}\subseteq \bel{\Psi\contract A}$, by \KCon{8}, as required.
\end{pproof}

\vspace{1em}

%=====================================================

\NaryRep*

\begin{pproof}
From postulate to construction: Assume that $\oplus$ satisfies \FacMin. We must specify, for each TPO profile  $\mathbf{P}=\langle\preccurlyeq_1,\ldots,\preccurlyeq_n\rangle$, a sequence $\langle a_{\mathbf{P}} (i)\rangle_{i\in\mathbb{N}}$ such that:

\begin{itemize}

\item[(1)] $\emptyset \neq a_{\mathbf{P}}(i) \subseteq I$ for each $i\in\mathbb{N}$ 

\item[(2)] $\oplus=\oplus_a$ 

\end{itemize}
We specify $a_{\mathbf{P}}(i)$ as follows.  Assuming $\oplus\mathbf{P}$ is represented by the ordered partition $\langle S_1,\ldots,S_m \rangle $, we have

\vspace{1em}

\begin{centering}

$a_{\mathbf{P}}(i) = \{ j\in I\mid \min(\preccurlyeq_j, \bigcap_{k<i}S^c_k)\subseteq S_i\} $

\end{centering}

\vspace{1em}
 
\noindent We prove each of (1) and (2) in turn.
\begin{itemize}

\item[(1)] By construction, $a_{\mathbf{P}}(i) \subseteq I$. So we simply need to show $\emptyset \neq a_{\mathbf{P}}(i) $. By the definition of  $\langle S_1,\ldots,S_m \rangle $ we know that $S_i=\min(\preccurlyeq_{\oplus}, \bigcap_{k<i} S^c_k)$. Then, by \FacMin, $S_i=\bigcup_{j\in X}\min(\preccurlyeq_j, \bigcap_{k<i} S^c_k)$ for some $X\subseteq I$. Since $S_i\neq\emptyset$ and hence $X\neq\emptyset$, we know that $S_i$ contains $\min(\preccurlyeq_j, \bigcap_{k<i} S^c_k)$ for at least one $j$. So $\emptyset \neq a_{\mathbf{P}}(i)$, as required.

\item[(2)] Let $\mathbf{P}=\langle\preccurlyeq_1,\ldots,\preccurlyeq_n\rangle$ be a given profile, and assume $\langle T_1,\ldots,T_l \rangle $ is the ordered partition representing  $\oplus_a(\mathbf{P})$ so, for each $i=1,\ldots,l$,

\vspace{1em}

\begin{centering}

$T_i=\bigcup_{j\in a_{\mathbf{P}}(i)}\min(\preccurlyeq_j, \bigcap_{k<i} T^c_k)$ 

\end{centering}

\vspace{1em}

We show by induction on $i$ that $T_i=S_i$ for all $i$. Fix $i$ and assume, for induction, $T_k=S_k$ for all $k<i$. 
\begin{itemize}

\item[-] $T_i\subseteq S_i$: By the inductive hypothesis we know $T_i=\bigcup_{j\in a_{\mathbf{P}}(i)}\min(\preccurlyeq_j, \bigcap_{k<i} S^c_k)$ and by construction $\min(\preccurlyeq_j, \bigcap_{k<i} S^c_k)\subseteq S_i$ for all $ j\in a_{\mathbf{P}}(i)$. Thus $T_i\subseteq S_i$,  as required.

\item[-] $S_i\subseteq T_i$: By the definition of  $\langle S_1,\ldots,S_m \rangle $ we know that $S_i=\min(\preccurlyeq_{\oplus}, \bigcap_{k<i} S^c_k)$. By \FacMin, there exists $X\subseteq I$ such that $S_i  = \bigcup_{j\in X} \min(\preccurlyeq_j, \bigcap_{k<i} S^c_k )$. For each $j\in X$, $\min(\preccurlyeq_j, \bigcap_{k<i} T^c_k)\subseteq S_i$ and so $ j\in a_{\mathbf{P}}(i)$.  
Thus $X\subseteq a_{\mathbf{P}}(i)$. So $S_i\subseteq \bigcup_{j\in a_{\mathbf{P}}(i)}\min(\preccurlyeq_j, \bigcap_{k<i} S^c_k)$. Then, by the inductive hypothesis, $S_i\subseteq \bigcup_{j\in a_{\mathbf{P}}(i)}\min(\preccurlyeq_j, \bigcap_{k<i} T^c_k)$, i.e.~$S_i\subseteq T_i$, as required.
\end{itemize}

\end{itemize}
\noindent From construction to postulate: Let $\oplus=\oplus_a$ for some given $a$. To show \FacMin, by Proposition \ref{prop:FSyntSem}, it suffices to show:

\begin{tabbing}
\=BLAHBLAI\= ooooo \= \kill
\> \FacPref \> Assume that $x_1$ to $x_n$ are s.t.~$x_i \preccurlyeq_i y$. Then  there\\
\> \>  exists  $j\in I$ s.t. \>  \\[0.1cm]
\>  \>  (i)  \> if  $x_j \prec_j y$, then  $x_j \prec_\oplus y$, and \\[0.1cm]
\> \>  (ii)  \> if  $x_j \preccurlyeq_j y$, then  $x_j \preccurlyeq_\oplus y$\\[-0.25em]
\end{tabbing}

\noindent Assume that $x_i\preccurlyeq_i y$, for all $i$ and suppose, for contradiction, that, for all $i$, either (i) $x_i \prec_i y$ and  $y  \preccurlyeq_{\oplus_a} x_i$ or (ii) $x_i \preccurlyeq_i y$ and $y  \prec_{\oplus_a} x_i$. Then we must have   $y\preccurlyeq_{\oplus_a} x_i$, for all $i$. Let $\langle T_1,\ldots,T_m \rangle $ be the ordered partition representing  $\oplus_a(\mathbf{P})$ and $j$ be such that $y\in T_j$. By the definition of $\oplus_a$, we know that 
$T_j=\bigcup_{l\in a_{\mathbf{P}}(j)}\min(\preccurlyeq_l, \bigcap_{k<l} T^c_k)$ and so $y\in \min(\preccurlyeq_l, \bigcap_{k<l} T^c_k)$ for some $l\in a_{\mathbf{P}}(j) $.

Since $y\preccurlyeq_{\oplus_a} x_i$ for all $i$, we know $x_i\in \bigcap_{k<j} T^c_k$, for all $i$, and so, in particular, $x_l\in \bigcap_{k<j} T^c_k$. Therefore, by the minimality of $y$, we have $y \preccurlyeq_l x_l$. Since we assumed  $x_i\preccurlyeq_i y$, for all $i$, we must then have $y\sim_l x_l$ and so also $x_l\in \min(\preccurlyeq_l, \bigcap_{k<l} T^c_k)$. Hence $x_l \in T_j$. Since by assumption $y\in T_j$, we therefore have $y\sim_{\oplus_a} x_l$. However, $y\sim_{\oplus_a} x_l$ and $y\sim_l x_l$ jointly contradict the assumption that , for all $i$, either (i) $x_i \prec_i y$ and  $y  \preccurlyeq_{\oplus_a} x_i$ or (ii) $x_i \preccurlyeq_i y$ and $y  \prec_{\oplus_a} x_i$.
\end{pproof}

\vspace{1em}

%=====================================================

\FSyntSem*

\begin{pproof}
From \FacPref~to \FacMin: Let $X= \{j\in I\mid \min(\preccurlyeq_j, S)\subseteq\min(\preccurlyeq_\oplus, S)\}$. Claim: $\min(\preccurlyeq_\oplus, S)=\bigcup_{j\in X}\min(\preccurlyeq_j, S)$.  $\min(\preccurlyeq_\oplus, S)\supseteq\bigcup_{j\in X}\min(\preccurlyeq_j, S)$ holds trivially by the definition of $X$, so we just need to establish $\min(\preccurlyeq_\oplus, S)\subseteq\bigcup_{j\in X}\min(\preccurlyeq_j, S)$. Assume that $y\in \min(\preccurlyeq_\oplus, S)$. Assume for contradiction that $y\notin\bigcup_{j\in X} \min(\preccurlyeq_j, S)$. We will show that, for all $j\in I$, there exists $x_j$ such that 
$x_j\preccurlyeq_j y$ but either (i) $x_j \prec_j y$ and  $y  \preccurlyeq_{\oplus} x_j$ or (ii) $x_j \preccurlyeq_j y$ and $y  \prec_{\oplus} x_j$,
contradicting \FacPref~and hence allowing us to conclude $y\in\bigcup_{j\in X} \min(\preccurlyeq_j, S)$, as required.
\begin{itemize}

\item[-] $j\in X$:  From $y\notin\bigcup_{j\in X} \min(\preccurlyeq_j, S)$, we have the fact that, for all $j\in X$, $y\notin \min(\preccurlyeq_j, S)$. Then, for each $j\in X$, there exists $x_j\in S$ such that $x_j\prec_j y$ and, since $y\in \min(\preccurlyeq_\oplus, S)$,  $y\preccurlyeq_\oplus x_j$.

\item[-] $j\notin X$: By the definition of $X$, for each $j\notin X$, there exists $x_j\in S$, such that $x_j\in \min(\preccurlyeq_j, S)-\min(\preccurlyeq_\oplus, S)$, hence $x_j\preccurlyeq_j y$ and, furthermore, since $y\in \min(\preccurlyeq_\oplus, S)$, we also have $ y\prec_\oplus x_j$. 

\end{itemize}
From \FacMin~to \FacPref: Suppose \FacMin~holds, but, for contradiction, \FacPref~does not. Then there exists a set $S=\{x_i\mid i\in I\}\cup\{y\}$ such that
$x_i\preccurlyeq_i y$ but either (i) $x_i \prec_i y$ and  $y  \preccurlyeq_{\oplus} x_i$ or (ii) $x_i \preccurlyeq_i y$ and $y  \prec_{\oplus} x_i$, for all $i$. 
Since, for each $i$, we have $y\preccurlyeq_\oplus x_i$, this means that $y\in \min(\preccurlyeq_\oplus, S)$. By \FacMin, there exists $X\subseteq I$ such that $\min(\preccurlyeq_{\oplus},  S )= \bigcup_{j\in X} \min(\preccurlyeq_j,  S )$. It then follows that $y\in\min(\preccurlyeq_j, S)$ for some $j\in X$. So $y\preccurlyeq_j x_j$. But we know that $x_j\preccurlyeq_j y$, so $x_j\sim_j y$ and $x_j\in\min(\preccurlyeq_j, S)$, and so, since $\min(\preccurlyeq_{\oplus},  S )= \bigcup_{j\in X} \min(\preccurlyeq_j,  S )$, we have $x_j\in \min(\preccurlyeq_\oplus, S)$. But from  $x_j\sim_j y$ and either  (i) $x_j \prec_j y$ and  $y  \preccurlyeq_{\oplus} x_j$ or (ii) $x_j \preccurlyeq_j y$ and $y  \prec_{\oplus} x_j$, we must have  $y\prec_\oplus x_j$, contradicting $x_j\in\min(\preccurlyeq_\oplus, S)$. Hence \FacPref~holds, as required.   
\end{pproof}

\vspace{1em}

%=====================================================

\PAR*

\begin{pproof} 
In fact, \FacPref~is not required in its full strength for the result. Only a particular consequence of it is needed 

\begin{tabbing}
    \=BLAHBLAII\=\kill
\>   \SPUPlus  \> Assume that $x_1$ to $x_n$ are s.t.~$x_i \prec_i y$. Then\\
\> \>   there  exists $j\in I$ s.t.~$x_j \prec_\oplus y$  \\[-0.25em]
\end{tabbing}
\vspace{-1em}

\noindent (See proof of Proposition  \ref{prop:SemBreakdown} below for the derivation of \SPUPlus.) 

We need to show that if $\oplus$ satisfies \SPUPlus~and \PPAR, then we have $\preccurlyeq_{\oplus}=\preccurlyeq_{\STQ}$. Assume that $\preccurlyeq_{\oplus}$ and $\preccurlyeq_{\STQ}$ are respectively  represented by $\langle S_1, S_2,\ldots, S_m \rangle$ and $\langle T_1, T_2,\ldots, T_n \rangle$. We will prove, by induction on $i$, that, $\forall i$,  $S_i=T_i$. Assume $S_j=T_j$, $\forall j < i$. We must show $S_i=T_i$.

\begin{itemize}

\item[(i)] Regarding $S_i\subseteq T_i$: Let $x\in S_i$, so that $x\preccurlyeq_{\oplus} y$, $\forall y\in \bigcap_{j<i} S^{\mathsf{c}}_j$. Assume for reductio that $x\notin T_i$. Since $x\in S_i$, we know that $x\in\bigcap_{j<i} S^{\mathsf{c}}_j=\bigcap_{j<i} T^{\mathsf{c}}_j$. Hence, since $x\notin T_i$ and, by construction of $\preccurlyeq_{\STQ}$, for all $k\in I$, there exists $y_k\in \bigcap_{j<i} T^{\mathsf{c}}_j=\bigcap_{j<i} S^{\mathsf{c}}_j$ such that $y_k \prec_{k} x$. Then, by \SPUPlus, there exists $l$ such that  $y_l \prec_{\oplus} x$,  contradicting  $x\preccurlyeq_{\oplus} y$, $\forall y\in \bigcap_{j<i} S^{\mathsf{c}}_j$. Hence $x\in T_i$, as required.

\item[(ii)] Regarding $T_i\subseteq S_i$: Let $x\in T_i$. Then, by construction of $\preccurlyeq_{\STQ}$, we have $x\in\bigcup_{k\in I}\min(\preccurlyeq_k, \bigcap_{j<i} T^{\mathsf{c}}_j)$. Assume for reductio that $x\notin S_i$. We know that $x\in \bigcap_{j<i} T^{\mathsf{c}}_j$, so by the inductive hypothesis, $x\in \bigcap_{j<i} S^{\mathsf{c}}_j$. From this and $x\notin S_i$ we know that there exists $y\in S_i$, such that   $y\prec_{\oplus} x$.  Then from \PPAR, for all $k \in I$ there exists $z_k\in S_i$  such that  $z_k\prec_{k} x$. But this contradicts $x\in\bigcup_{k\in I}\min(\preccurlyeq_k, \bigcap_{j<i} T^{\mathsf{c}}_j)$. Hence $x\in S_i$, as required. 

\end{itemize}
\end{pproof}

\vspace{1em}

%=====================================================

\PARSB*

\begin{pproof} 
\begin{itemize}

\item[(i)] From \PPAR~to \SB: Assume that $x\prec_\oplus y $ for all $x \in S^c$, $y \in S$. 
We must show that  $\bigcup_{i\in I}\min( \preccurlyeq_{i},S) \subseteq\min( \preccurlyeq_{\oplus},S)$. So assume $x\in \bigcup_{i\in I}\min( \preccurlyeq_{i},S) $ but, for contradiction, $x\notin  \min( \preccurlyeq_{\oplus},S)$. Then $y\prec_\oplus x$, for some $y\in S$. From the latter, by \PPAR, we know that, for each $i \in I$ there exists $z_i$ such that    $y \sim_{\oplus} z_i$ and  $z_i \prec_i x$.  Given our initial assumption, since $y\in S$, we can deduce from $y \sim_{\oplus} z_i$, for all $i\in I$, that $z_i\in  S $, for all $i\in I$. But this, together with $z_i \prec_i x$,  for all $i\in I$,  contradicts  $x\in \bigcup_{i\in I}\min( \preccurlyeq_{i},S) $. Hence $x\in  \min( \preccurlyeq_{\oplus},S)$, as required.

\item[(ii)] From \SB~to \PPAR: Suppose \PPAR~does not hold,  i.e.~$\exists x,y$ such that $x\prec_{\oplus} y$  but for some $i \in I$ there does not exist $z$ such that $x \sim_{\oplus} z$ and  $z \prec_i y$.We will show that \SB~fails, i.e.~ that $\exists S\subseteq W$, such that $x\prec_{\oplus} y$  for all $x \in S^c$, $y \in S$, but  $\bigcup_{i\in I}\min( \preccurlyeq_{i},S)\nsubseteq\min( \preccurlyeq_{\oplus},S)$. Let $S=\{w\mid x\preccurlyeq_{\oplus} w\}$ (so that $S^c=\{w\mid w\prec_{\oplus} x\}$). Clearly $x\in  S $ and, from $x\prec_{\oplus} y$, we know that $y\in S $ but $y\notin\min(\preccurlyeq_{\oplus}, S )$. Hence, to show $\bigcup_{i\in I}\min( \preccurlyeq_{i},S)\nsubseteq\min( \preccurlyeq_{\oplus},S)$ and therefore that \SB~fails, it suffices to show $y\in\min(\preccurlyeq_{i}, S )$. But if $y\notin\min(\preccurlyeq_{i}, S )$, then  $z\prec_{i} y$ for some $z\in S $, i.e. some $z$, such that  $x\preccurlyeq_{\oplus} z$. Since $\preccurlyeq_{\oplus}$ is a TPO we may assume $x\sim_{\oplus} z$. This contradicts our initial assumption that for no $z$ do we have $x\sim_{\oplus} z$ and $z\prec_{i} y$. Hence $y\in\min(\preccurlyeq_{i}, S )$, as required.

\end{itemize}
\end{pproof}

\vspace{1em}

%=====================================================

\SemBreakdown* 

\begin{pproof} From \FacPref~to  \SPUPlus: Assume that $x_1$ to $x_n$ are s.t.~$x_i \prec_i y$. Then, $x_1$ to $x_n$ are s.t.~$x_i \preccurlyeq_i y$. So, by part (i) of \FacPref, there  exists  $j\in I$ s.t.,  if  $x_j \prec_j y$, then  $x_j \prec_\oplus y$ and therefore  there exists $j\in I$ s.t.  $x_j \prec_\oplus y$, as required.

From \FacPref~to  \WPUPlus: Assume that $x_1$ to $x_n$ are s.t.~$x_i \preccurlyeq_i y$. Then, by part (ii) of \FacPref, there  exists  $j\in I$ s.t.,  if  $x_j \preccurlyeq_j y$, then  $x_j \preccurlyeq_\oplus y$ and therefore  there exists $j\in I$ s.t.  $x_j \preccurlyeq_\oplus y$, as required.   
\end{pproof}

\vspace{1em}

%=====================================================

\SWPUSynt*

\begin{pproof}
Regarding the equivalence between \SPUPlus~and \UBO:
\begin{itemize}

\item[(i)] From \SPUPlus~to \UBO: Assume that $y\notin \bigcup_{i\in I} \min(\preccurlyeq_i, S)$. We need to show that $y\notin \min(\preccurlyeq_\oplus, S)$. If $y\notin S$, then we are done. So assume $y\in S$. From  $y\notin \bigcup_{i\in I} \min(\preccurlyeq_i, S)$, for all $i\in I$, there exists $x_i\in S$ such that $x_i\prec_i y$. By  \SPUPlus, we have $x_i\prec_\oplus y$ for some $j$. Hence $y\notin \min(\preccurlyeq_\oplus, S)$, as required.

\item[(ii)] From \UBO~to \SPUPlus: Suppose that, for all $i\in I$, there exists $x_i\in S$ such that $x_i\prec_i y$. let $S=\{y\}\cup\{x_i\mid i\in I\}$. Since $x_i\prec_i y$ for all $i\in I$, we know  that $y\notin \bigcup_{i\in I} \min(\preccurlyeq_i, S)$. Then, by \UBO, $y\notin \min(\preccurlyeq_\oplus, S)$. Hence, there exists $j$ such that $x_j\prec_\oplus y$, as required.

\end{itemize}
Regarding the equivalence between \WPUPlus~and \LBO:
\begin{itemize}

\item[(i)] From \WPUPlus~to \LBO: Suppose \WPUPlus~holds and assume for contradiction that, for all $i\in I$, there exists $x_i\in \min(\preccurlyeq_i, S)$, such that $x_i\notin \min(\preccurlyeq_\oplus, S)$. Since $\preccurlyeq_\oplus$ is a TPO, this means that there exists $y\in S$, such that, for all $i\in I$, $y\prec_\oplus x_i$.Then, by \WPUPlus, there must exist $j$ such that $y\prec_j x_j$, contradicting $x_j\in \min(\preccurlyeq_j, S)$.

\item[(ii)] From \LBO~to \WPUPlus: Suppose $x_1,\ldots, x_n$ are such that $x_i\preccurlyeq_i y$. If $y= x_j$ for some $j$, then we are done, so assume that, for all $i\in I$, $y\neq x_j$. Assume for contradiction that, for all $i\in I$, $y\prec_\oplus x_i$. Let $S=\{y\}\cup\{x_i\mid i\in I\}$. Then $ \min(\preccurlyeq_\oplus, S) =\{y\}$
. By \LBO, we have $\min(\preccurlyeq_j, S)\subseteq \{y\}$, for some $j\in I$, so $y\prec_j x_j$. Contradiction.   
\end{itemize}
\end{pproof}

\vspace{1em}

%=====================================================

%=====================================================

\Flattest*

\begin{pproof}
Let profile $\mathbf{P} = (\preccurlyeq_1, \ldots, \preccurlyeq_n)$ be given. 
Let $\langle T_1,\ldots, T_m\rangle$ be the ordered partition corresponding to $\preccurlyeq_{\STQ}$. Let $\langle S_1,\ldots, S_n\rangle$ be the ordered partition corresponding to $\preccurlyeq_{\oplus}$. We must show that $\langle T_1,\ldots, T_m\rangle \sqsupseteq \langle S_1,\ldots, S_n\rangle$.

If $T_i=S_i$ for all $i$, then we are done. So let $i$ be minimal such that $T_i\neq S_i$. We must show $S_i\subset T_i$. So let $y\in S_i$ and assume, for contradiction, that  $y\notin T_i$. We know that $T_i\neq \emptyset$, since, otherwise, $\bigcup_{j < i} T_j = W$, hence $\bigcup_{j < i} S_j = W$ and so $S_i=\emptyset$, contradicting $S_i\neq T_i$. So let $x\in T_i$. Then $x \prec_{\STQ} y$. So, by \PPAR, for each $k$, $\exists z_k$ such that $z_k\sim_{\STQ} x$ (i.e.~$z_k\in T_i$) and $z_k\prec_k y$. Since $\oplus$ satisfies \SPUPlus, it follows that $z_s\prec_{\oplus} y$ for some $s$. But then, since $y\in S_i$, $z_s\in \bigcup_{j < i} S_j = \bigcup_{j < i} T_j$, contradicting $z_s\in T_i$. Hence $y\in T_i$ and so $S_i\subset T_i$, as required.   
\end{pproof}

\vspace{1em}

%=====================================================

\CtoCPack*

\begin{proof} 
The proof is obtained by adapting the proof of Proposition 10 of \cite{DBLP:journals/ai/BoothC19} and adding a few small steps. Note that we only require $\oplus$ to satisfy the weak principles \SPU~and \WPU, not even  \SPUPlus~or   \WPUPlus, let alone \PPAR. Assume that we have $S=\{A_1,\ldots,A_n\}$ and that \CConR{1}--\CConR{4} are satisfied:
\begin{itemize}

\item[ (a)]  \CConRn{1}: Assume that $x,y \in\mods{\bigwedge \neg S}$. We must show that $x \preccurlyeq_{\Psi \odiv S}   y$ iff  $x \preccurlyeq_\Psi y$. Note first that, from \CConR{1}, it follows that (1) $x\preccurlyeq_{\Psi\contract A_i} y$ iff $x\preccurlyeq_{\Psi} y$, for all $i$. Regarding the left-to-right direction of the equivalence: Assume (2) $y\prec_\Psi x$. We want to show $y\prec_{\Psi \odiv S}   x$.  From (1) and (2), we recover (3) $y\prec_{\Psi\contract A_i} x$ for all $i$. From (3), by \SPU, it follows that $y \prec_{\Psi \odiv S}   x$, as required. Regarding the right-to-left-direction: Assume (4) $x\preccurlyeq_{\Psi} y$. We want to show $x\preccurlyeq_{\Psi \odiv S}   y$. From (1) and (4), we recover (5) $x\preccurlyeq_{\Psi\contract A_i} y$, for all $i$. From (4) and (5), by \WPU, it follows that $x\preccurlyeq_{\Psi \odiv S}   y$, as required. 

\item[(b)] \CConRn{2}: Similar proof to the one given in (a).

\item[(c)] \CConRn{3}:  Let $x \in\mods{\bigwedge \neg S}$, $y \notin\mods{\bigwedge \neg S}$  and $x \prec y$. We must show that $x \prec_{\Psi \odiv S}   y$. For all $i$, either (i) $x, y\in\mods{\neg A_i}$ or (ii) $x\in\mods{\neg A_i}$ and $y\in\mods{A_i}$. Either way, we recover $x\prec_{\Psi \contract A_i} y$, for all $i$: in case (i), by  \CConR{1}, and in case (ii), by \CConR{3}. From this, by \SPU, we then obtain $x\prec_{\Psi \odiv S}    y$, as required.

\item[(d)]  \CConRn{4}: Similar proof to the one given in (c), using \CConR{4} rather than \CConR{3} and   \WPU ~rather than \SPU.    

\end{itemize}
\end{proof}

\vspace{1em}

%=====================================================

\end{document}